\documentclass{article}

\usepackage{amsmath,amsfonts,bm}

\def\eqref#1{equation~\ref{#1}}

\def\1{\bm{1}}

\DeclareMathAlphabet{\mathsfit}{\encodingdefault}{\sfdefault}{m}{sl}
\SetMathAlphabet{\mathsfit}{bold}{\encodingdefault}{\sfdefault}{bx}{n}

\PassOptionsToPackage{numbers, compress}{natbib}

     \usepackage[preprint]{neurips_2020}

\usepackage[utf8]{inputenc} 
\usepackage[T1]{fontenc}    
\usepackage{hyperref}       
\usepackage{url}            
\usepackage{booktabs}       
\usepackage{amsfonts}       
\usepackage{nicefrac}       
\usepackage{microtype}      

\usepackage{hyperref}
\usepackage{url}
\usepackage{graphicx, wrapfig, lipsum}
\usepackage{subcaption}
\usepackage{booktabs}
\usepackage{amsthm}
\usepackage{tabularx}
\newcolumntype{C}{>{\Centering\arraybackslash}X}
\newtheorem{theorem}{Theorem}[section]

\newtheorem{lemma}[theorem]{Lemma}
\title{Formatting Instructions For NeurIPS 2020}

\title{SEMULATOR: Emulating the Dynamics of Crossbar Array-based Analog Neural System with Regression Neural Networks}

\author{Chaeun Lees and Seyoung Kim \\
Department of Material Science and Engineering\\
Pohang University of Science and Technology\\
Pohang-si, Republic of Korea \\
\texttt{\{chaeunl, kimseyoung\}@postech.ac.kr} \\
}

\begin{document}

\maketitle

\begin{abstract}
As deep neural networks require tremendous amount of computation and memory, analog computing with emerging memory devices is a promising alternative to digital computing for edge devices. However, because of the increasing simulation time for analog computing system, it has not been explored. To overcome this issue, analytically approximated simulators are developed, but these models are inaccurate and narrow down the options for peripheral circuits for multiply-accumulate operation (MAC). In this sense, we propose a methodology, SEMULATOR (SiMULATOR by Emulating the analog computing block) which uses a deep neural network to emulate the behavior of crossbar-based analog computing system. With the proposed neural architecture, we experimentally and theoretically shows that it emulates a MAC unit for neural computation. In addition, the simulation time is incomparably reduced when it compared to the circuit simulators such as SPICE.
\end{abstract}

\section{Introduction}
\label{sec_introduction}

For the application of neural networks on edge devices, analog computing using emerging memory devices with crossbar array architecture is promising in comparison with digital computing in the aspect of its plentiful representation capability with efficiency. This is because analog computing is based on the motivation that analog computation (e.g., peripheral circuit for accumulation) and memory unit (e.g., crossbar array, memory devices) have infinitely continuous states whereas digital computing has only a few states. Along with it, the time complexity of multiplication and accumulation in the crossbar array is $O(1)$. The nature of analog computing requires accurate circuit simulators such as SPICE, but it takes much time to simulate a cumbersome system like neural networks. Considering that research on deep learning enters golden age thanks to the development of GPU, the low-speed simulators for analog computing system is one of main factors that hinders research on this field. 

To enhance simulation speed, approximated simulators, or analytical models, replace accurate circuit simulators. They simplify the computation and memory unit into an analytic function rather than dozens of non-analytic functions, enabling the analytic functions to be merged into high-level programming language and make use of machine learning frameworks which are run on GPU environment. However, they have three problems: Analytically modeled units result in overestimating the issues in analog computing~\cite{chakraborty2020resistive, chakraborty2020geniex} due to inaccurate modeling. On top of that, they require human experts to approximately model each unit into simple one, so that this methodology necessitates tremendous amount of human resources. If it is hard to approximate with an analytic function, then the units are modeled with multiple analytic functions with conditional statement such as a spline~\cite{gupta2017steam}, which is reluctant in the GPU environment. 

In another way, statistical model~\cite{mcconaghy2011high} and neural networks~\cite{nguyen2019deep} are used to emulate analog computation units or analog circuits. However, these works cannot deal with large input dimensions, while analog computing blocks which are consisted of hundreds of analog computation and memory units have thousands of input parameters. With digital computing system, a digital computing block such as logic gates is relatively smaller than that of analog, so that small fully connected neural network (FCNN) emulates their behavior with little error~\cite{abrishami2019csm, abrishami2020nn}. Instead of emulating the analog computing blocks, neural networks are used to find specific hyper-parameters of them or analog circuit system such as the arrangement of units or blocks~\cite{zhang2019circuit}, layout~\cite{hsieh2019learning}, and non-linear factor of crossbar arrays~\cite{chakraborty2020geniex}. On these specific tasks, neural networks and deep learning find  reasonable solutions, but there are little research on emulating an analog computing block to build a full analog neural system.


SEMULATOR proposes a methodology and neural network architecture for emulating the analog computing block. The neural network architecture for the purpose of emulating analog computing blocks includes feature extractors that fit with crossbar array architecture and circuit equation solver for extracted features of crossbar arrays (memory units) and circuit parameters of peripheral circuits (computing units). This architecture enables neural networks to emulate analog neural system, resolving the curse of dimension. In addition, while previous works only may use analog-digital-converters (ADCs) for analog computing, this architecture does not restrict the types of analog computing units. To verify the feasibility of our proposed neural architecture, we define the upper bound of error between SPICE simulation results and predicted values of neural networks, which gives conditions for training the neural networks such as required number of training data and epochs. Along with machine learning frameworks, as the architecture is trained on the frameworks, it can reduce simulation time astoundingly with little error compared to the result of SPICE. 

\section{Related Works}
\label{sec_related_works}

\subsection{Simulator for Analog Neural Computing}
\label{sec_simulator}
To systematically design analog computing neural networks, research groups have suggested simulators by approximating memory device and peripheral circuits. According to the type of approaches to model crossbar array~\cite{chakraborty2020resistive}, analytical inference model~\cite{jain2018rxnn, roy2020txsim} depicts the model as matrix and its linear operation. Analytical training models support differential operation for backward path using deep learning frameworks~\cite{he2019noise}. Some of them includes non-idealities in memory devices, so that assumes more complicated models~\cite{peng2020dnn+, maltae2020aihwkit}. However, because of the limitation of linearly-approximated analytical model, statistical model~\cite{chakraborty2020geniex} use neural networks to find the ratio of non-linear to linear model.

\subsection{Emulation by neural networks}
\label{sec_nn_for_DES}

In the aspect of circuit and device, statistical models~\cite{mcconaghy2011high,gupta2017steam} are widely used to emulate them by regression. As deep neural networks (DNNs) have the ability to solve ordinary or partial differential equations~\cite{dockhorn2019discussion, hsieh2019learning}, DNNs are used as a model to emulate the circuit behavior which can be represented by differential equations~\cite{nguyen2018transient, nguyen2019deep}. For the purpose of enhancing simulation speed, \cite{abrishami2019csm, abrishami2020nn} suggests to use simple neural networks for emulating logic gates, and then simulate a digital system with the neural networks. In another application of emulation, neural networks emulate the biological models~\cite{wang2019massive}.

\begin{figure}[!t]
\begin{center}
    \includegraphics[width=0.95\linewidth]{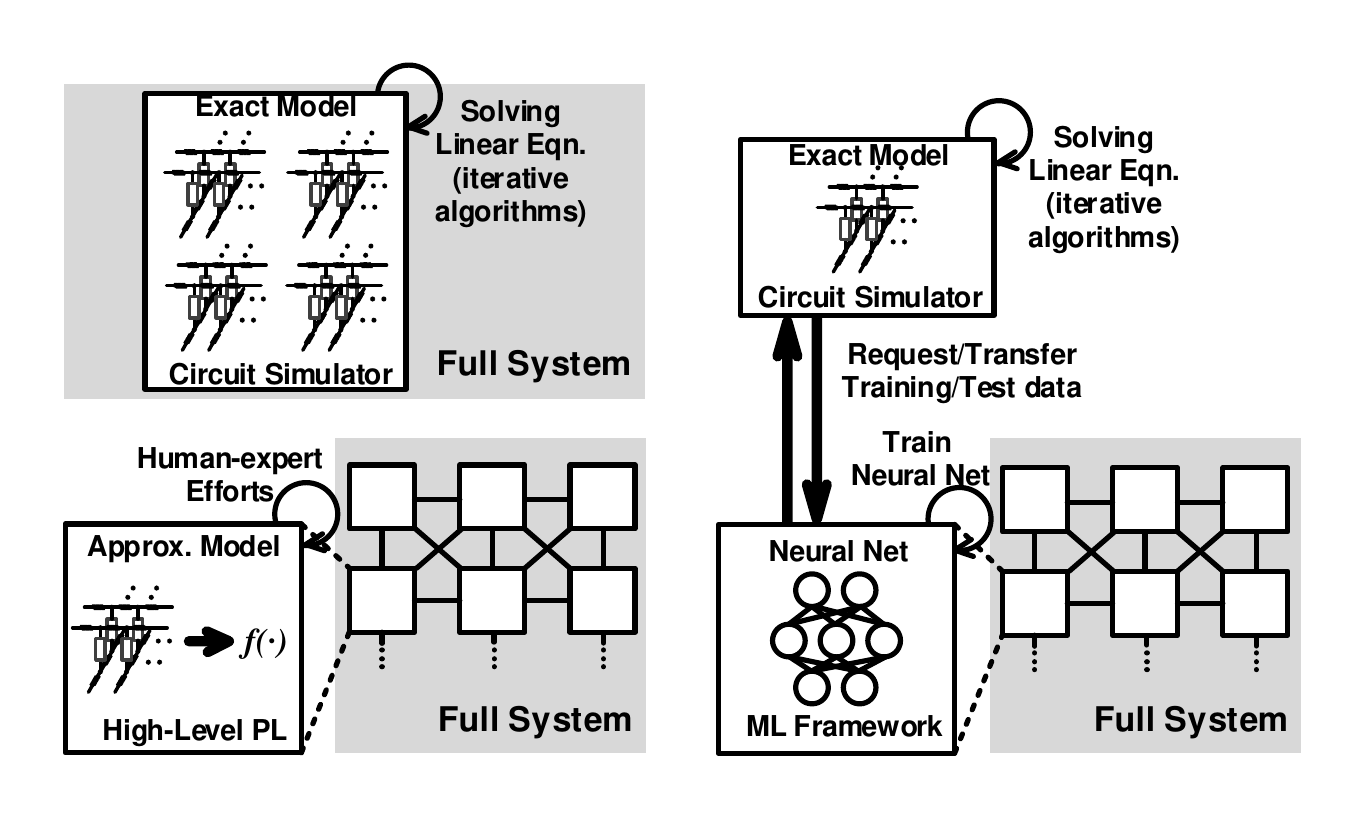}
\end{center}
\caption{Comparison of three methodologies for simulating analog computing system. (upper left) General circuit simulators such as SPICE. (lower left) Analytically approximated models by human experts. (right) Proposed methodologies.}
\label{fig_methodologies}
\end{figure}

\section{SEMULATOR}
\label{sec_semulator}
\subsection{Methodologies}
\label{subsec_methodologies}

SEMULATOR aims at emulating the response of circuits using crossbar array by neural networks. As shown in Figure~\ref{fig_methodologies}, general circuit simulators builds circuit equations and find solutions by iterative algorithms such as Newton-Raphson method. However, due to exploding circuit parameters and computation, the simulators require much time and resources. To relieve this issue, analog computation and memory units are approximated by analytical models which are developed by human experts~\cite{jain2018rxnn, roy2020txsim, he2019noise, peng2020dnn+, maltae2020aihwkit}: Crossbar arrays of analog memory units are linearly approximated, and other memory units such as memory devices and computation units are non-linearly approximated. This methodology incorporates the approximated models into deep learning frameworks to pursue efficient simulations. However, this approach 
narrows down the choice for analog computation and memory units. Because of underlying complicated physical model, the units cannot be represented by an analytic function, $f(\mathbf{x}),~\text{where}~\mathbf{x} \in \mathbb{R}^n$. Instead, it is natural to represent them as a set of non-analytic functions, $f(\mathbf{x})=f_{1}(\mathbf{x_1}),...,f(\mathbf{x})=f_{m}(\mathbf{x}_{m}),~\text{where}~\mathbf{x_1} \cup ... \cup \mathbf{x_m}=\mathbf{x} \in \mathbb{R}^n \text{~and~} \mathbf{x_i} \cap \mathbf{x_j} = \emptyset~(i \neq j)$~\cite{gupta2017steam}. 

On the other hand, emulation-based methodologies~\cite{abrishami2019csm, abrishami2020nn, chakraborty2020geniex} emulate computation or memory units such as logic gates and the non-linear factor in crossbar arrays. SEMULATOR also is involved in this methodology, but it emulates a complete computational analog computing block such as MAC unit or analog neuron. When an analog computing block is designed in circuit level, it requests the data for training and test. The neural network that we propose in Section~\ref{subsec_nnarch} is trained by regression to emulate the non-analytic function.
 
\subsection{Neural Network Architecture for SEMULATOR}
\label{subsec_nnarch}
\begin{wrapfigure}{r}{0.3\textwidth}
\vspace{-10pt}
\begin{center}
    \hspace{-0pt}
    \includegraphics[width=0.25\textwidth]{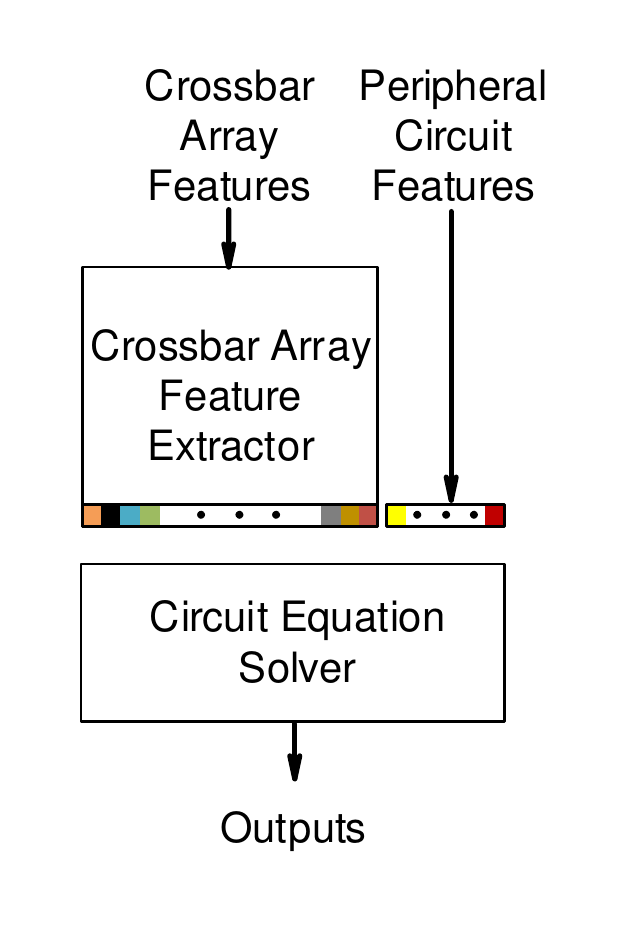}
    \vspace{-20pt}
\end{center}
\caption{The overall neural network architecture for SEMULATOR.}
\vspace{-0pt}
\label{fig_overall_arch}
\end{wrapfigure}
The neural network to emulate the analog computation block is mainly composed of two parts: Crossbar array feature extractor and circuit equation solver. Figure~\ref{fig_overall_arch} shows the schematic of our proposed architecture for SEMULATOR. The extracted hidden features and features of analog computing units which in general accumulates the current from analog memory units are concatenated. The concatenated features are interpreted as boundary conditions for circuit differential equations with respect to time domain, and so the circuit equation solver is expected to solve differential equations to find solutions of specific time slot. As neural networks have the ability to find solutions of specific partial or ordinary differential equations~\cite{dockhorn2019discussion, hsieh2019learning}, we use FCNN or Neural ODE~\cite{chen2018neural} as circuit equation solvers.

In the crossbar array, the behavior of each cell is mainly determined by the features of the cell such as applied voltage to the cell and the conductance of memory device in the cell. Therefore, the behavior of the crossbar array, $C(\mathbf{X})$, is formulated as follows:
\begin{equation}
C(\mathbf{X}) = C(d_{1,1,1}(\mathbf{x_{1,1,1}}),..,d_{r,c,t}(\mathbf{x_{r,c,t}}))
\end{equation}
,where $\mathbf{X}=(\mathbf{x}_{ijk}) \in \mathbb{R}^{f \times r \times c \times t}$ is a tensor and $\mathbf{x}_{i,j,k} \in \mathbb{R}^f.$ $f,r,c, \text{and}~t$ indicate the number of features in a cell, row, column, and tiles. According to the choice of a cell, the features of a cell are differed. For instance, 1R cell~\cite{meena2014overview} has two features, applied voltage to the cell and the conductance of the cell, and 1T1R cell~\cite{meena2014overview} has additional 1 transistor whose  features such as threshold voltage and W/L ratio.

In most cases of the crossbar array, $d_{i,j,k}(\cdot)$ has nearly same form, $d(\cdot)$, because the structure of each cell is similar, which means that the function, $d(\cdot)$, is shared along all cells along the row, column, and tile. As shown in Figure~\ref{fig_conv4xbar}, the architecture of crossbar array has structural and physical similarity with the convolutional neural network (CNN). At the first layer, the unit size filters which has unit width and length isolate the features of the cell to emulate its behavior, $d(\cdot)$. In the deeper layers, for instance, if the crossbar array is designed to accumulate currents along column, then the filters learn the column-wise locality. In addition, deeper layers allow networks to learn other non-linear behaviors in $C(\cdot)$. 

\begin{figure}[t]
\begin{center}
\includegraphics[width=0.98\linewidth]{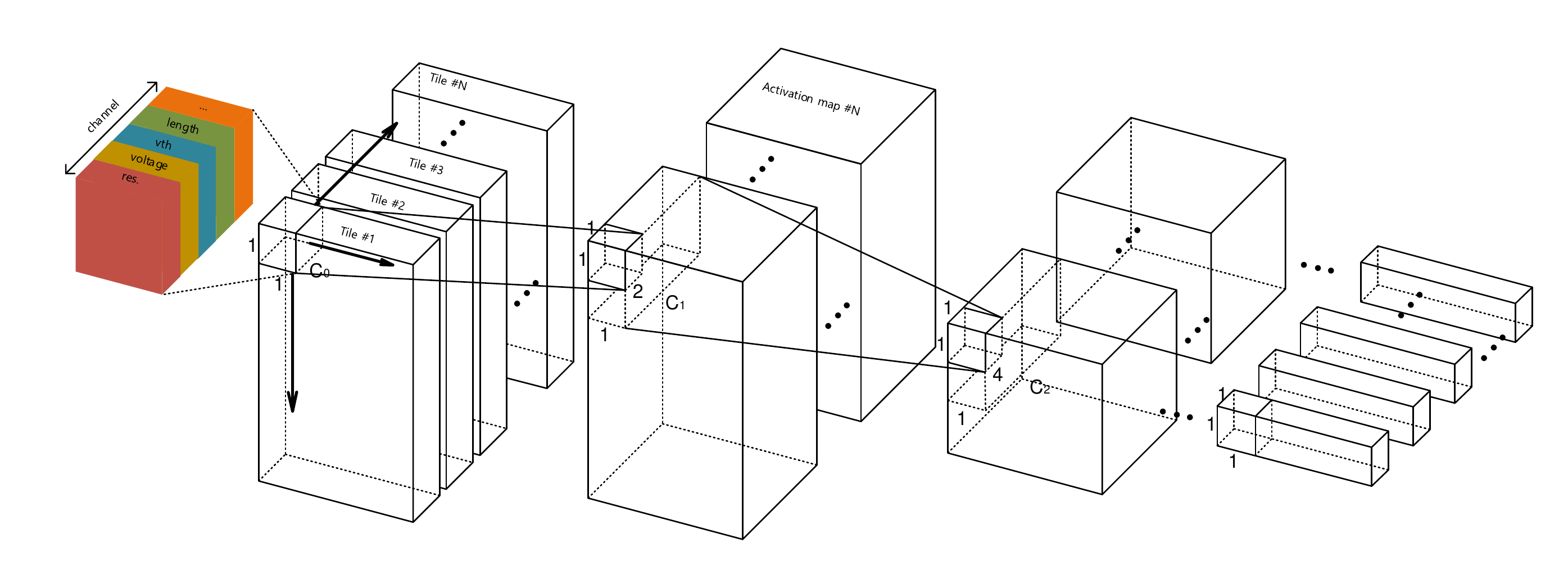}
\end{center}
\caption{The architecture of Conv4Xbar. In a feature map, a box with solid line indicates a output feature corresponding to a filter and a box with dotted line indicates a filter for the stage. Input feature maps has $N$ depths which corresponds to the number of total tiles and $C_{0}$ channels whose values are features of a cell in the crossbar array such as resistance, applied voltage, and so on.}
\label{fig_conv4xbar}
\end{figure}

To guarantee these assumptions, we adopt 3D convolutional neural networks (3D-CNN) whose depth is 1 and the size of length or width increase through the deeper layer. We denote this type of neural networks as Conv4Xbar as shown in Figure~\ref{fig_conv4xbar}. In the aspect of optimizing neural networks, it has advantages along with the structural similarity. The filters are trained scanning through all cells in the tiles, which leads to make the filters more generalized. The generalization power of this architecture stems from the belief that incorporating the prior knowledge improves generalization~\cite{dugas2009incorporating}. In addition, the filters help neural networks not to suffer from the curse of dimension because the filters are shared for all tiles of input feature maps or activation maps.

\section{Emulating analog computing block}
\label{sec_emulating_acb}

\subsection{statistical verification}
\label{sec_statistical_verification}
As analog computing has continuous states, the emulated neural networks are not free of error though it is believed small enough. This is because the error can be amplified in deep neural networks as it propagates through deeper layer. \cite{liao2018defense} addresses error amplification effect of adversarial attacks. To mitigate the effect, the neural networks should be trained under the upper bound of error. In this sense, it is necessary to evaluate training neural networks by observing training or validation. Theorem~\ref{thm_err_and_loss} gives the upper bound of the error by directly observing the mean-squared-error (MSE) loss, where the upper bound of loss is represented by the significant bit of error, $s$, and the probability, $p$, of the condition that the error is less than half of significant bit. It is also possible to directly evaluate the condition, $P_{(X,Y) \sim D}(|Y-f(X)| < 0.5 \cdot 10^{-s}) > p$, but if target neural networks has multiple outputs, it is controversial to evaluate them respectively.

\begin{theorem}
Let $f$ be a regression neural network and $D$ is the distribution of training data. To satisfy the condition, $P_{(X,Y) \sim D}(|Y-f(X)| < 0.5 \cdot 10^{-s}) > p$, the upper bound of mean-squared error $\mathbb{E}_{(X,Y) \sim D}(|Y-f(X)|^2)$ is $\Big(\cfrac{1}{2} \cdot \cfrac{10^{-s}}{erf^{-1}(p)}\Big)$ for the given $s$ and $p$.
\label{thm_err_and_loss}
\end{theorem}

\begin{proof}
Let the error, $Z=Y-f(X)$. By Lemma~\ref{lemma_error}, $Z \sim \mathcal{N}(0, \sigma^2)$. Then, we have
\begin{align*}
P_{Z}(|Z| < 10^{-s}) & = P_{Z}(Z < 10^{-s}) - P_{Z}(Z < -10^{-s}) = \Phi\Big(\frac{10^{-s}}{\sigma}\Big) - \Phi\Big(-\frac{10^{-s}}{\sigma}\Big) \\
                     & = \frac{1}{2}\Big( 1 + \text{erf} \Big(\frac{10^{-s}}{\sqrt{2\sigma^{2}}}\Big)\Big) - \frac{1}{2}\Big( 1 + \text{erf} \Big(-\frac{10^{-s}}{\sqrt{2\sigma^{2}}}\Big)\Big) = \text{erf} \Big(\frac{10^{-s}}{\sqrt{2\sigma^{2}}}\Big) > p
\end{align*}
,wehre $\Phi(x)=\frac{1}{\sqrt{2\pi}}\int_{-\infty}^{x}e^{-\frac{t^{2}}{2}}\,dt = \frac{1}{2}\Big(1+\text{erf}\Big(\frac{x}{\sqrt{2}}\Big)\Big)$. As we know that $\mathbb{E}_{Z}({|Z|}^{2})=\sigma^{2}$, we can find the upper bound of $\mathbb{E}_{(X,Y) \sim D}(|Y-f(X)|^2)$ to be $\frac{1}{2} \Big(\frac{10^{-s}}{\text{erf}^{-1}(p)}\Big)^{2}$. 
\end{proof}

\begin{lemma}
If a neural network $f$ is trained for regression data $D$ with the mean squared error, then the error, $Y-f(X)=Z \sim  \mathcal{N}(0, \sigma^{2})$.
\label{lemma_error}

\end{lemma}

\subsection{generalization}
\label{sec_generalization}

\begin{table}[t]
\caption{Experimental results of analog computing blocks. Inputs have four axes, each of whom corresponds to (hardware parameters, tile, row address, column address). Outputs spread from voltage, current, and so on, which is the output of analog computing blocks. MAE indicates the mean-average-error between SPICE results and predicted values.}
\centering
\begin{tabular}{ l c c c c}
  \toprule
  \textbf{Computing Blocks} & \textbf{Inputs (C,D,H,W)}  & \textbf{Outputs (O)} & \textbf{Data (N)} & \textbf{MAE}\\ 
  \midrule
  RRAM+PS32~\cite{kim2019vcam} & (2,4,64,2) & 1 voltage & 50,000 & 0.981(mV) \\
  RRAM+PS32~\cite{kim2019vcam} & (2,2,64,8) & 4 voltage & 50,000 & 0.955(mV) \\
  \bottomrule
\end{tabular}
\label{table_generalization}
\end{table}

\begin{figure}[t]
    \includegraphics[width=\textwidth]{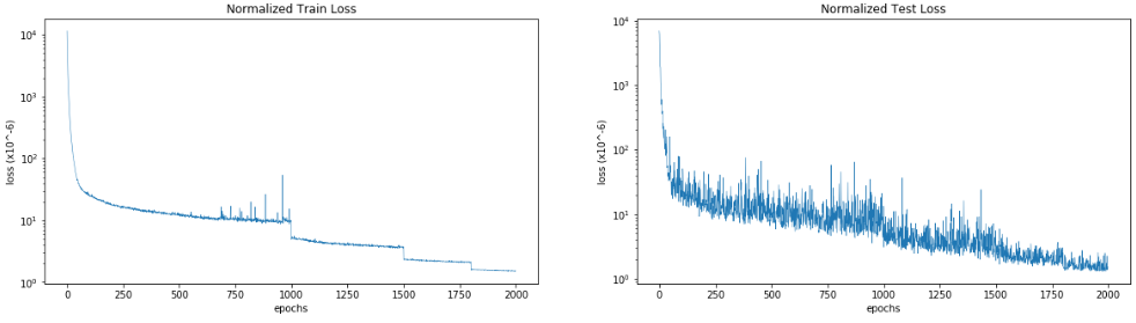}
\caption{Train and test loss of RRAM32+PS32 case. Learning rate is halved at 1000, 1500, and 1800 epochs.}
\label{fig_rram_ps32_loss}
\end{figure}

Table~\ref{table_generalization} shows the experimental results for various analog computing blocks with SEMULATOR. We use SPYCE~\cite{chaeun2020spyce} to generate data for SEMULATOR. RRAM+PS32 use 1T1R cell as analog memory units and customized analog circuit, PS32, for analog computing units. In this experiment, we assume 2 cases: 1 MAC unit for an analog computing block and 4 MAC units for an analog computing block. These cases differ from the hardware design choice. As shown in Figure~\ref{fig_rram_ps32_loss}, the train loss decreases with little gap between test loss, which means that the neural network is neither overfitted nor underfitted. Following Theorem~\ref{thm_err_and_loss}, we set the significant bit of error, $s$, to be 3 and the probability, $p$, to be 0.3, so that the upper bound of loss is about $6.7 \cdot 10^{-6}$. 

Along with loss, in a more generalized case, Figure~\ref{fig_rram_ps32_heatmap} shows the outputs of trained neural networks of the case, RRAM+32 and 1 voltage output, when the parameters of a cell are changed and other parameters are randomly chosen: weight axis corresponds to the normalized conductance (G) and activation to the normalized applied voltage (V). Considering that the output voltage response ($\Delta O$) of 1T1R cell is generally as followed,
\begin{align*}
\Delta O \approx G_{const},~&if~V < V_{const}\\
\Delta O \approx \frac{1}{2}k(V-V_{const})^{\alpha} ,~&otherwise
\label{equation_1t1r}
\end{align*}
the neural network properly generalizes the behavior of the non-linear cell. As we discussed before in Section~\ref{subsec_methodologies}, approximated models use the non-analytical functions, which leads to inaccurate model.

\begin{figure}[t]
    \begin{subfigure}{0.32\linewidth}
        \includegraphics[width=\linewidth]{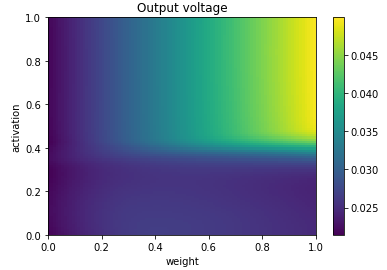}
    \end{subfigure}
    \begin{subfigure}{0.32\linewidth}
        \includegraphics[width=\linewidth]{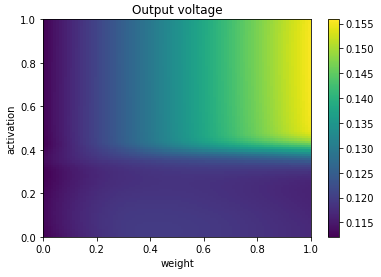}    
    \end{subfigure}
    \begin{subfigure}{0.32\linewidth}
        \includegraphics[width=\linewidth]{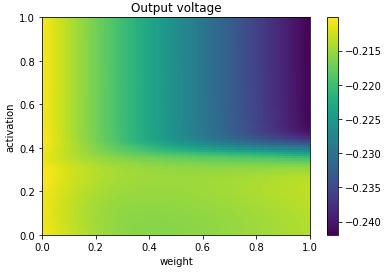}
    \end{subfigure}
\caption{For two normalized input parameters, voltage and conductance, the heatmap shows the output of RRAM+PS32. The two figures on the left are the results of the cell corresponding to a positive weight and right one to a negative weight.}
\label{fig_rram_ps32_heatmap}
\end{figure}

\section{Conclusion}
As neural networks have the ability to solve differential equations and emulate their behavior, we propose SEMULATOR which includes a methodology and neural architecture for crossbar array-based analog computing system to enhance simulation speed and accurate simulation. Therefore, it may adopt various types of peripheral circuits, and allow researcher not to simulate whole system on the classical circuit simulators where simulating cumbersome system is impossible. In addition, we suggest an upper bound of loss where the significant bit is specified.     

\textbf{Data Requirements and Loss}\\
Figure~\ref{fig_data_requirements} shows that as the number of data increases train loss decreases. It indicates that tens thousands of data are required to reduce the error and avoid underfitting. Considering that generating data requires additional resources such as licences for circuit simulators and CPU server, it is promising to suggest an algorithm to reduce the number of required data. 

\begin{figure}[h]
    \centering
    \includegraphics[width=0.7\linewidth]{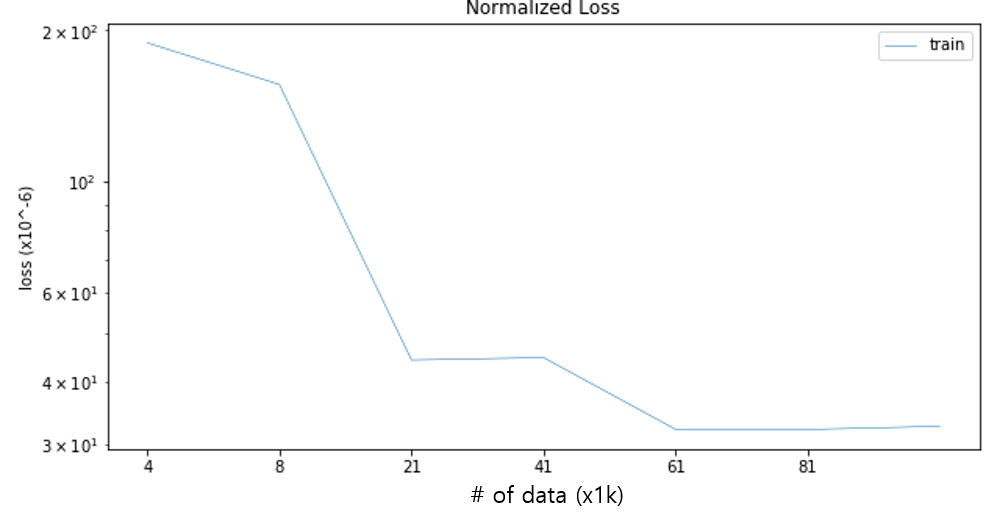}
\caption{The relationship between the number of data and train loss}
\label{fig_data_requirements}
\end{figure}

\bibliography{references}

\newpage

\appendix
\section{The neural network architectures for SEMULATOR.}

\begin{table}[h]
\caption{The neural network architectures for SEMULATOR. The format is as follows: Conv3d(in\_channels, out\_channels, kernel\_size=(D,H,W), stride\_size=(D,H,W), padding=(0,0,0)), Linear(in\_features, out\_features), where D,H, and W stands for depth, height, and width.}
\begin{tabularx}{\textwidth}{l X}
  \toprule
  Computing Blocks & Neural Network Architecture \\ 
  \midrule
  RRAM+PS32 &  Conv3d(2,16,(1,1,1),(1,1,1))-CELU-Conv3d(16,8,(1,2,1),(1,2,1))-CELU-Conv3d(8,4,(1,4,1),(1,4,1))-CELU-Conv3d(4,32,(1,8,1),(1,8,1))-CELU-Conv3d(32,32,(1,1,2),(1,1,1))-CELU-Linear(128,32)-CELU-Linear(32,16)-CELU-Linear(16,1) \\
  RRAM+PS32 & Conv3d(2,16,(1,1,1),(1,1,1))-CELU-Conv3d(16,8,(1,2,1),(1,2,1))-CELU-Conv3d(8,4,(1,4,1),(1,4,1))-CELU-Conv3d(4,32,(1,8,1),(1,8,1))-CELU-Conv3d(32,32,(1,1,2),(1,1,1))-CELU-Linear(256,32)-CELU-Linear(32,16)-CELU-Linear(16,1) \\
  \bottomrule 
\end{tabularx}
\label{table_neural_arch}
\end{table}

\newpage

\appendix
\section{The additional experimental results.}

\begin{figure}[h]
\centering
\begin{subfigure}{0.6\linewidth}
    \includegraphics[width=\linewidth]{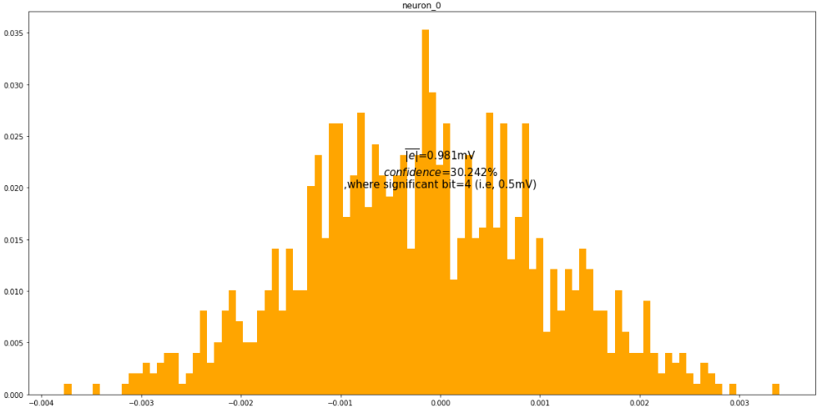} \\
\end{subfigure}
\begin{subfigure}{\linewidth}
    \includegraphics[width=\linewidth]{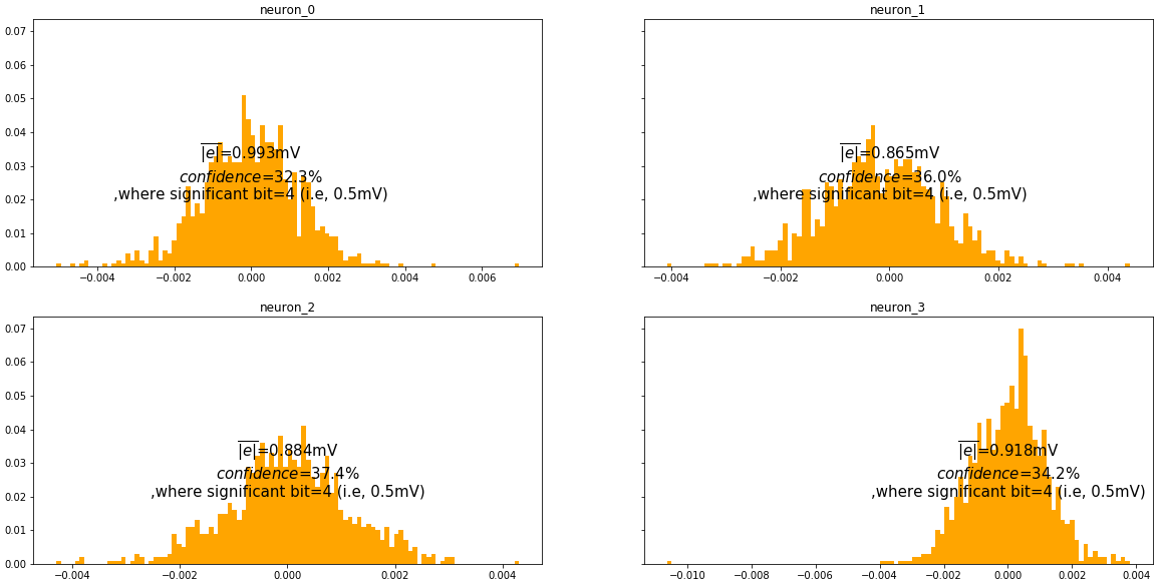}
\end{subfigure}
\caption{The error distribution of test data for Table~\ref{table_generalization}.}
\label{fig_rram_ps32_err_distrib}
\end{figure}

\end{document}